\documentclass[letterpaper, 10 pt, conference]{ieeeconf}  

\IEEEoverridecommandlockouts                              

\usepackage{graphicx}                   
\usepackage{mathptmx}                   
\usepackage{times}                      
\usepackage{amsmath}                    
\usepackage{amssymb}                    
\usepackage{algorithm}
\usepackage{algorithmic}
\usepackage{subfig}
\usepackage{balance}
\graphicspath{{figures/}}

\DeclareMathOperator*{\argmax}{argmax}

\newtheorem{theorem}{Theorem}
\newcommand{\comment}[1]{}

\title{\LARGE \bf
Bellman Gradient Iteration for Inverse Reinforcement Learning
}

\author{Kun Li$^{1}$, Yanan Sui$^{1}$, Joel W. Burdick$^{1}$
\thanks{*This work was
supported by the National Institutes of Health, NIBIB.} \thanks{$^{1}$Kun Li, Yanan Sui
and Joel W. Burdick are with Department of Mechanical and Civil Engineering, California Institute of
Technology, Pasadena, CA 91125, USA {\tt\small kunli@caltech.edu}
}%
}

\begin{document}

\maketitle
\thispagestyle{empty}
\pagestyle{empty}

\begin{abstract}
  This paper develops an inverse reinforcement learning algorithm aimed at recovering a reward
  function from the observed actions of an agent. We introduce a strategy to flexibly handle
  different types of actions with two approximations of the Bellman Optimality Equation, and a
  Bellman Gradient Iteration method to compute the gradient of the Q-value with respect to the
  reward function.  These methods allow us to build a differentiable relation between the Q-value
  and the reward function and learn an approximately optimal reward function with gradient methods.
  We test the proposed method in two simulated environments by evaluating the accuracy of different
  approximations and comparing the proposed method with existing solutions. The results show that
  even with a linear reward function, the proposed method has a comparable accuracy with the
  state-of-the-art method adopting a non-linear reward function, and the proposed method is more
  flexible because it is defined on observed actions instead of trajectories 
\end{abstract}

\section{introduction}
\label{irl::intro}
In many problems, the actions of an agent in an environment can be modeled as a Markov Decision
Process, where the environment decides the states and transitional probabilities, and the agent
decides its own reward function based on the preferences over the states and takes actions
accordingly. Since the agent's reward function determines its actions, it is possible to estimate
the state preferences from the observed actions, hence the inverse reinforcement learning problem.

This problem arises in many applications. For example, in robot learning by demonstration
\cite{irl:rlfd}, an operator may manipulate an object based on knowledge and preference of the
object, like which object states are achievable and which object states are desired. By learning the
knowledge and preference from the observed operator motion, a robot can manipulate the object in an
appropriate way. Another application is analyzing a person's physical wellness from daily
observation of motions. Assuming the person's actions are based on self-evaluation of physical
limitations, the change of such limitations (a potential sign of health problems) can be reflected
by long-term monitoring of the subject's motion and estimated via inverse reinforcement learning
(IRL).

To solve the problem, it is critical to model the relation between the agent's actions and the
reward functions. Since an action depends on both the immediate reward and future rewards, existing
solutions model either a relation between the actions and the value function \cite{irl::maxentropy,
irl::irl2}, or a relation between the actions and the Q-function \cite{irl::bayirl,
irl::irl1,irl::maxmargin, irl::subgradient}. To efficiently compute the optimal reward function, the
gradient of the optimal value function and the optimal Q-function with respect to the reward
function parameter is necessary, but the optimal value function and optimal Q function are
non-differentiable with respect to the rewards, and existing solutions adopt different
approximations to alleviate the problem.

This paper introduces two approximations of the Bellman Optimality Equation to make the optimal
value function and the optimal Q-function differentiable with respect to the reward function, and
proposes a Bellman Gradient Iteration method to compute the gradients efficiently. The approximation
level can be adjusted with a parameter to adapt to different types of action preferences, like
preferring an action leading to an optimal future path, or an action leading to uncertain future
paths. To the best of our knowledge, no previous work computes the gradients by modeling the
relation between motion and reward in a differentiable way.

The paper is organized as follows. We review existing work on inverse reinforcement learning in
Section \ref{irl::related}, and formulate the gradient-based method in Section \ref{irl::irl}. We
introduce Bellman Gradient Iteration method to compute the gradients in Section \ref{irl::gradient}.
Several experiments are shown in Section \ref{irl::experiments}, with conclusions in Section
\ref{irl::conclusions}.

\section{Related Works}
\label{irl::related}
The Inverse Reinforcement Learning problem is first formulated in \cite{irl::irl1}, where the agent
observes the states resulting from an assumingly optimal policy, and tries to learn a reward
function that makes the policy better than all alternatives. Since the goal can be achieved by
multiple reward functions, this paper tries to find one that maximizes the difference between the
observed policy and the second best policy. This idea is extended by \cite{irl::maxmargin}, in the
name of max-margin learning for inverse optimal control. Another extension is proposed in
\cite{irl::irl2}, where the goal is not to recover the actual reward function, but to find a reward
function that leads to a policy equivalent to the observed one, measured by the total reward
collected by following that policy.

Since a motion policy may be difficult to estimate from observations, a behavior-based method is
proposed in \cite{irl::maxentropy}, which models the distribution of behaviors as a maximum-entropy
model on the amount of reward collected from each behavior. This model has many applications and
extensions. For example, Nguyen et al. \cite{irl::sequence} consider a sequence of changing reward
functions instead of a single reward function. Levine et al. \cite{irl::gaussianirl} and Finn et al.
\cite{irl::guidedirl} consider complex reward functions, instead of linear ones, and use Gaussian
process and neural networks, respectively, to model the reward function. Choi et al.
\cite{irl::pomdp} consider partially observed environments, and combines partially observed Markov
Decision Process with reward learning. Levine et al.  \cite{irl::localirl} model the behaviors
based on the local optimality of a behavior, instead of the summation of rewards.  Wulfmeier et al.
\cite{irl::deepirl} use a multi-layer neural network to represent nonlinear reward functions.

Another method is proposed in \cite{irl::bayirl}, which models the probability of a behavior as the
product of each state-action's probability, and learns the reward function via maximum a posteriori
estimation. However, due to the complex relation between the reward function and the behavior
distribution, the author uses computationally expensive Monte-Carlo methods to sample the
distribution. This work is extended by \cite{irl::subgradient}, which uses sub-gradient methods to
reduce the computations.  Another extensions is shown in \cite{irl::bayioc}, which tries to find a
reward function that matches the observed behavior. For motions involving multiple tasks and varying
reward functions, methods are developed in \cite{irl::multirl1} and \cite{irl::multirl2}, which try
to learn multiple reward functions. 

Our method uses gradient methods like \cite{irl::subgradient}, but we introduce two approximation
methods that improve the flexibility of motion modeling, and a Bellman Gradient Iteration algorithm
that computes the gradient of the optimal value function and the optimal Q-function with respect to
the reward function accurately and efficiently. 

\section{Inverse Reinforcement Learning}
\label{irl::irl}
\subsection{Markov Decision Process}
A Markov Decision Process is described with the following variables:
\begin{itemize}
  \item $S=\{s\}$, a set of states
  \item $A=\{a\}$, a set of actions
  \item $P_{ss'}^a$, a state transition function that defines the probability that state $s$ becomes
    $s'$ after action $a$.
  \item $R=\{r(s)\}$, a reward function that defines the immediate reward of state $s$.
  \item $\gamma$, a discount factor that ensures the convergence of the MDP over an infinite
    horizon.
\end{itemize}

A motion can be represented as a sequence of state-action pairs:
\[\zeta=\{(s_i,a_i)|i=0,\cdots,N_\zeta\}\]
where $N_\zeta$ denotes the length of the motion.

One key problem is how to choose the action in each state, or the policy, $\pi(s)\in A$, a mapping
from states to actions. This problem can be handled by reinforcement learning algorithms, by
introducing the value function $V(s)$ and the Q-function $Q(s,a)$, described by the Bellman Equation
\cite{irl::rl}:
\begin{align}
  &V^\pi(s)=\sum_{s'|s,\pi(s)}P_{ss'}^{\pi(s)}[r(s')+\gamma*V^\pi(s')],\\
  &Q^\pi(s,a)=\sum_{s'|s,a}P_{ss'}^a[r(s')+\gamma*V^\pi(s')]
\end{align}
where $V^\pi$ and $Q^\pi$ define the value function and the Q-function under a policy $\pi$.

For an optimal policy $\pi^*$, the value function and the Q-function should be maximized on every
state. This is described by the Bellman Optimality Equation \cite{irl::rl}:
\begin{align}
  &V^*(s)=\max_{a\in A}\sum_{s'|s,a}P_{ss'}^a[r(s')+\gamma*V^*(s')],\\
  &Q^*(s,a)=\sum_{s'|s,a}P_{ss'}^a[r(s')+\gamma*\max_{a'\in A}Q^*(s',a')].
  \label{equation:bellmanequation}
\end{align}

With the optimal value function, $V^*(s)$, and Q-function, $Q^*(s,a)$, the action $a$ for state $s$
can be chosen in multiple ways. For example, the agent may choose $a$ in a stochastic way:
\[P(a|s)\propto Q^*(s,a), a \in A\]
where the agent's probability to choose action $a$ in state $s$ is proportional to the optimal Q
value $Q^*(s,a)$.

\subsection{Motion Modeling}
Assuming the reward function $r$ is parameterized by $\theta$, we model $P(\zeta|\theta)$ based on
the optimal Q-value of each state-action pair of $\zeta=\{(s_i,a_i)|i=0,\cdots,N_\zeta\}$:
\begin{equation}
  P(\zeta|\theta)=\prod_{(s,a)\in \zeta} P((s,a)|\theta)
  \label{equation:problem}
\end{equation}
where \begin{equation}P((s,a)|\theta)=\frac{\exp{b*Q^*(s,a)}}{\sum_{\hat{a}\in
A}\exp{b*Q^*(s,\hat{a})}}\label{equation:motionmodel}\end{equation} defines the probability to
choose action $a$ in state $s$ based on the formulation in \cite{irl::bayirl}, and $b$
is a parameter controlling the degree of confidence in the agent's ability to choose actions based on
Q values. In the remaining sections, we use $Q(s,a)$ to denote the optimal Q-value of the
state-action pair $(s,a)$. Since $Q(s,a)$ depends on reward function $r$, it also depends on
$\theta$.

In this formulation, the inverse reinforcement learning problem is equivalent to maximum-likelihood
estimation of
$\theta$:
\begin{equation}
  \label{equation:maxtheta}
  \theta=\argmax_{\theta}\log{P(\zeta|\theta)}
\end{equation}
where the log-likelihood of $P(\zeta|\theta)$ is given by:
\begin{equation}
  L(\theta)=\sum_{(s,a)\in\zeta}(b*Q(s,a)-\log{\sum_{\hat{a}\in
  A}\exp{b*Q(s,\hat{a}))}}
  \label{equation:loglikelihood}
\end{equation}
and the gradient of the log-likelihood is given by:
\begin{align}
  \nabla L(\theta)&=\sum_{(s,a)\in\zeta}(b*\nabla Q(s,a)\nonumber\\
  &-b*\sum_{\hat{a}\in
  A}P((s,\hat{a})|r(\theta))\nabla Q(s,\hat{a})).
  \label{equation:loglikelihoodgradient}
\end{align}

If we can compute the gradient of the Q-function $\nabla Q=\frac{\partial Q}{\partial
\theta}=\frac{\partial Q}{\partial r}\cdot \frac{\partial r}{\partial \theta}$, we can use gradient
methods to find a locally optimal parameter value:
\begin{equation}
  \label{equation:gradientascent}
  \theta=\theta+\alpha*\nabla L(\theta)
\end{equation}
where $\alpha$ is the learning rate. When the reward function is linear, the cost function is convex
and the global optimum can be achieved.  The standard way to compute the optimal Q-value is with the
following Bellman Equation of Optimality \cite{irl::rl} with Equation
\eqref{equation:bellmanequation}.

However, the Q-value in Equation \eqref{equation:bellmanequation} is non-differentiable with respect
to $r$ or $\theta$ due to the max operator. Its gradient $\nabla Q(s,a)$ cannot be computed in a
conventional way, and the sub-gradient method in \cite{irl::subgradient} cannot compute the
gradients everywhere in the parameter space. We propose a method called Bellman Gradient Iteration
to solve the problem.

\section{Bellman Gradient Iteration}
\label{irl::gradient}
To handle the non-differentiable max function in Equation \eqref{equation:bellmanequation}, we
introduce two approximation methods.

\subsection{Approximation with a P-Norm Function}
The first approximation is based on a p-norm:
\begin{equation}
  \label{equation:pnorm}
  \max(a_0,\cdots,a_n)\approx (\sum_{i=0}^n a_i^k)^{\frac{1}{k}}
\end{equation}
where $k$ controls the level of approximation, and we assume all the values $a_0,\cdots,a_n$ are
positive. When $k=\infty$, the approximation becomes exact. In the remaining section, we refer to
this method as p-norm approximation.

Under this approximation, the Q-function in Equation \eqref{equation:bellmanequation} can be rewritten
as:
\begin{equation}
  Q_p(s,a)=\sum_{s'|s,a}P_{ss'}^a[r(s')+\gamma*(\sum_{a'\in A}Q_p^k(s',a'))^{1/k}].
  \label{equation:qpnorm}
\end{equation}

From Equation \eqref{equation:qpnorm}, we construct an approximately optimal value function with
p-norm approximation:
\begin{equation}
  V_p(s)=(\sum_{a\in A}Q_p^k(s,a))^{1/k}.
  \label{equation:vpnorm}
\end{equation}

Using Equations \eqref{equation:qpnorm} and \eqref{equation:vpnorm}, we build an approximate
Bellman Optimality Equation to find the approximately optimal value function and Q-function:
\begin{align}
  &Q_p(s,a)=\sum_{s'|s,a}P_{ss'}^a[r(s')+\gamma*V_p(s')]\label{equation:qbellpnorm},\\
  &V_p(s)=(\sum_{a\in A}(
  \sum_{s'|s,a}P_{ss'}^{a}[r(s')+\gamma*V_p(s')]))^k)^{1/k}\label{equation:vbellpnorm}.
\end{align}

Taking derivative on both sides of Equation \eqref{equation:vpnorm} and Equation
\eqref{equation:qbellpnorm}, we construct a Bellman Gradient Equation to compute the gradients
of $V_p(s)$ and $Q_p(s,a)$ with respect to reward function parameter $\theta$:
\begin{align}
  &\frac{\partial V_p(s)}{\partial \theta}= \frac{1}{k}(\sum_{a\in A}Q_p^k(s,a))^{\frac{1-k}{k}}\sum_{a\in
  A}k*Q_p^{k-1}(s,a)\frac{\partial Q_p(s,a)}{\partial \theta}\label{equation:pnormvgrad},\\
  &\frac{\partial Q_p(s,a)}{\partial \theta}=\sum_{s'|s,a}P_{ss'}^a(\frac{\partial
  r(s')}{\partial \theta}+\gamma*\frac{\partial V_p(s')}{\partial
  \theta})\label{equation:pnormqgrad}.
\end{align}

For a p-norm approximation with non-negative Q-values, the gap between the approximate value function
and the optimal value function is a function of $k$:
\[g_p(k)=(\sum_{a'\in A} Q_p(s',a')^k)^{\frac{1}{k}} -\max_{a'\in A}Q_p(s',a').\]
The gap function $g_p(k)$ describes the error of the approximation, and it has two properties.
\begin{theorem}
  Assuming all Q-values are non-negative, $Q_p(s,a)\geq 0, \forall s,a$, the tight lower bound of
  $g_p(k)$ is zero: \[\inf_{\forall k\in R} g_p(k)=0.\]
\end{theorem}

\begin{proof}
    $\forall k \in R$, assuming $a_{max}=\argmax_{a'\in A}Q_p(s',a'),$
    \begin{align}
      g_p(k)&=(\sum_{a'\in A} Q_p(s',a')^k)^{\frac{1}{k}}-\max_{a'\in A}Q_p(s',a')\nonumber\\ 
      &=(\sum_{a'\in A/a_{max}}Q_p(s',a')^k+Q_p(s',a_{max})^k)^{\frac{1}{k}}-\max_{a'\in
      A}Q_p(s',a')\nonumber.
    \end{align}
  Since $Q_p(s,a)\geq 0\Rightarrow\sum_{a'\in A/a_{max}}Q_p(s',a')^k\geq 0$, 
    \begin{align}
      g_p(k)&\geq(Q_p(s',a_{max})^k)^{\frac{1}{k}}-\max_{a'\in
      A}Q_p(s',a')\nonumber\\
      &=Q_p(s',a_{max})-\max_{a'\in A}Q_p(s',a')=0\nonumber
    \end{align}
   When $k=\infty$:
    \begin{align}
      g_p(k)&=(\sum_{a'\in A} Q_p(s',a')^\infty)^{\frac{1}{\infty}} -\max_{a'\in A}Q_p(s',a')\nonumber\\
      &=\max_{a'\in A}Q_p(s',a')-\max_{a'\in A}Q_p(s',a')=0\nonumber
    \end{align}
\end{proof}
\begin{theorem}
Assuming all Q-values are non-negative, $Q_p(s,a)\geq 0, \forall s,a$, $g_p(k)$ is a decreasing
  function with respect to $k$: \[g'_p(k)\leq0, \forall k \in R.\]
\end{theorem}
\begin{proof}
    \begin{align}
      g'_p(k)&=\frac{1}{k}*(\sum_{a'\in A} Q_p(s',a')^k)^{\frac{1-k}{k}}*(\sum_{a'\in A}
      Q_p(s',a')^k\log(Q_p(s',a')))\nonumber\\
      &+(\sum_{a'\in A} Q_p(s',a')^k)^{\frac{1}{k}}\log(\sum_{a'\in A}
      Q_p(s',a')^k)\frac{1}{-k^2}\nonumber\\
      &=\frac{(\sum_{a'\in A} Q_p(s',a')^k)^{\frac{1}{k}}}{ k^2\sum_{a'\in A}Q_p(s',a')^k}(\sum_{a'\in
      A}Q_p(s',a')^kk\log(Q_p(s',a'))\nonumber\\
      &- \sum_{a'\in A}Q_p(s',a')^k \log(\sum_{a'\in A}Q_p(s',a')^k))\nonumber.
    \end{align}
  Since $k\log(Q_p(s',a'))\leq\log(\sum_{a'\in A}Q_p(s',a')^k)$:
  \[g_p'(k)\leq0.\]
\end{proof}

\subsection{Approximation with Generalized Soft-Maximum Function}
The second approximation is based on a generalized soft-maximum function:
\begin{equation}
  \label{equation:softmax}
  \max(a_0,\cdots,a_n)\approx\frac{\log(\sum_{i=0}^n\exp(ka_i))}{k}
\end{equation}
where $k$ controls the level of approximation. When $k=\infty$, the approximation becomes exact. In
the remaining sections, we refer to this method as g-soft approximation. 

Under this approximation, the Q-function in Equation \eqref{equation:bellmanequation} can be rewritten
as:
\begin{equation}
  Q_g(s,a)=\sum_{s'|s,a}P_{ss'}^a[r(s')+\gamma*\frac{\log{\sum_{a'\in A}\exp (kQ_g(s',a'))}}{k}].
  \label{equation:qgsoft}
\end{equation}

From Equation \eqref{equation:qgsoft}, we construct an approximately optimal value function with
g-soft approximation:
\begin{equation}
  V_g(s)=\frac{\log{\sum_{a\in A}\exp (kQ_g(s,a))}}{k}.
  \label{equation:vgsoft}
\end{equation}

With Equations \eqref{equation:qgsoft} and \eqref{equation:vgsoft}, we build an approximate
Bellman Optimality Equation to find the approximately optimal value function and Q-function:
\begin{align}
  &Q_g(s,a)=\sum_{s'|s,a}P_{ss'}^a[r(s')+\gamma*V_g(s')]\label{equation:qbellgsoft},\\
  &V_g(s)=\frac{\log{\sum_{a\in A}\exp
  (k(\sum_{s'|s,a}P_{ss'}^{a}[r(s')+\gamma*V_g(s'))}}{k})\label{equation:vbellgsoft}.
\end{align}

Taking derivative on both sides of Equations \eqref{equation:vgsoft} and
\eqref{equation:qbellgsoft}, we construct a Bellman Gradient Equation to compute the gradients of
$V_g(s)$ and $Q_g(s,a)$ with respect to the reward function parameter $\theta$:

\begin{align}
  &\frac{\partial V_g(s)}{\partial \theta}=\sum_{a\in A}\frac{\exp (kQ_g(s,a))}{\sum_{a'\in A}\exp
  (kQ_g(s,a'))} \frac{\partial Q_g(s,a)}{\partial
  \theta}\label{equation:gsoftvgrad},\\
  &\frac{\partial Q_g(s,a)}{\partial
  \theta}=\sum_{s'|s,a}P_{ss'}^a(\frac{\partial r(s')}{\partial \theta}+\gamma*\frac{\partial V_g(s')}{\partial
  \theta})\label{equation:gsoftqgrad}.
\end{align}

For a g-soft approximation, the gap between the approximate value function and the optimal value
function is:
\[g_g(k)=\frac{\log(\sum_{a'\in A}\exp(kQ_g(s',a')))}{k}-\max_{a'\in A}Q_g(s',a').\]
The gap has the following two properties.

\begin{theorem}
  The tight lower bound of $g_g(k)$ is zero: 
  \[\inf_{\forall k\in R} g_g(k)=0.\]
\end{theorem}
\begin{proof}
    $\forall k\in R$: assuming $a_{max}=\argmax_{a'\in A}Q_g(s',a'), $
    \begin{align}
      g_g(k)&=\frac{\log(\sum_{a'\in A}\exp(kQ_g(s',a')))}{k}-\max_{a'\in A}Q_g(s',a')\nonumber\\
      &=\frac{\log(\sum_{a'\in A/a_{max}}\exp(kQ_g(s',a'))+\exp(kQ_g(s',a_{max})))}{k}\nonumber\\&
      -\max_{a'\in A}Q_g(s',a')\nonumber\\
      &>Q_g(s',a_{max})-\max_{a'\in A}Q_g(s',a')=0\nonumber
    \end{align}
    When $k=\infty$,
    \begin{align}
      &\lim_{k\to\infty}(\frac{\log(\sum_{a'\in A}\exp(kQ_g(s',a')))}{k}-\max_{a'\in
      A}Q_g(s',a'))\nonumber\\
      &=\lim_{k\to\infty}(\frac{\log(\sum_{a'\in A}\exp(kQ_g(s',a')))}{k})-\max_{a'\in
      A}Q_g(s',a')\nonumber\\
      &=\max_{a'\in A}Q_g(s',a')-\max_{a'\in A}Q_g(s',a')=0\nonumber
    \end{align}
\end{proof}

\begin{theorem}
 $g_g(k)$ is a decreasing function with respect to $k$:
  \[g'_g(k)<0, \forall k \in R.\]
\end{theorem}
\begin{proof}
    \begin{align}
      g'_g(k)&=-\frac{\log(\sum_{a'\in A}\exp(kQ_g(s',a')))}{k^2}\nonumber\\
      &+\frac{\sum_{a'\in A}Q_g(s',a')\exp(kQ_g(s',a'))}{k\sum_{a'\in
      A}\exp(kQ_g(s',a'))}<0\nonumber
    \end{align}
    Since:
    \begin{align}
      &-\frac{\log(\sum_{a'\in A}\exp(kQ_g(s',a')))}{k^2}\nonumber\\
      &+\frac{\sum_{a'\in A}Q_g(s',a')\exp(kQ_g(s',a'))}{k\sum_{a'\in A}\exp(kQ_g(s',a'))}<0\nonumber\\
      &\Longleftrightarrow  \sum_{a'\in A}kQ_g(s',a')\exp(kQ_g(s',a'))<\nonumber\\
      &\sum_{a'\in A}\log(\sum_{a'\in A}\exp(kQ_g(s',a')))\exp(kQ_g(s',a'))\nonumber\\
      &\Longleftarrow kQ_g(s',a')<\log(\sum_{a'\in A}\exp(kQ_g(s',a')))\nonumber.
    \end{align}
\end{proof}

Based on the theorems, the gap between the approximated Q-value and the exact Q-value decreases with
larger $k$, thus the objective function in Equation \eqref{equation:loglikelihood} under
approximation will approach the true one with larger $k$.
\subsection{Bellman Gradient Iteration}
Based on the Bellman Equations \eqref{equation:qbellpnorm}, \eqref{equation:vbellpnorm},
\eqref{equation:qbellgsoft}, and \eqref{equation:vbellgsoft}, we can iteratively compute the value
of each state $V(s)$ and the value of each state-action pair $Q(s,a)$, as shown in Algorithm
\ref{alg:value}. In the algorithm, $apprxMax$ means a p-norm approximation of the max function for
the first method, and a g-soft approximation of the max function for the second method.

After computing the approximately optimal Q-function, with the Bellman Gradient Equation
\eqref{equation:pnormvgrad}, \eqref{equation:pnormqgrad}, \eqref{equation:gsoftvgrad}, and
\eqref{equation:gsoftqgrad}, we can iteratively compute the gradient of each state $\frac{\partial
V}{\partial \theta}$ and each state-action pair $\frac{\partial Q(s,a)}{\partial \theta}$ with
respect to the reward function parameter $\theta$, as shown in Algorithm \ref{alg:gradient}.  In the
algorithm, $\frac{\partial apprxMax}{\partial Q[s,a]}$ corresponds to the gradient of each
approximate value function with respect to the $Q$ function, as shown in Equation
\eqref{equation:pnormvgrad} and Equation \eqref{equation:gsoftvgrad}.

In these two approximations, the value of parameter $b$ depends on an agent's ability to choose
actions based on the Q values. Without application-specific information, we choose $b=1$ as an
uninformed parameter. Given a value for parameter $b$, the motion model of the agent is defined on
the approximated Q values, where the Q-value of a state-action pair depends on both the optimal path
following the state-action pair and other paths. When the approximation level $k$ is smaller, the
Q-value of a state-action pair relies less on the optimal path, and the motion model in Equation
\ref{equation:motionmodel} is similar to the model in \cite{irl::maxentropy}; When
$k\rightarrow\infty$, the Q-value approaches the standard Q-value, and the motion model is similar
to the model in \cite{irl::bayirl}.  By choosing different $k$ values, we can adapt the algorithm to
different types of motion models. 

With empirically chosen application-dependent parameters $k$ and $b$, Algorithm \ref{alg:value} and
Algorithm \ref{alg:gradient} are used compute the gradient of each Q-value, $Q[s,a]$, with respect
to the reward function parameter $\theta$, and learn the parameter with the gradient ascent method
shown in Equation \eqref{equation:loglikelihood} and Equation \eqref{equation:gradientascent}. With
the approximately optimal Q-function, the objective function is not convex, but a large $k$ will
make it close to a convex function, and a multi-start strategy handles local optimum. This process
is shown in Algorithm \ref{alg:irl}.
\begin{algorithm}[tb]
  \caption{Inverse Reinforcement Learning}
  \label{alg:irl}
\begin{algorithmic}[1]
  \STATE \textbf{Data}: {$S,A,P,\gamma$,k}
  \STATE \textbf{Result}: {Reward function}
  \STATE choose the number of random starts $n_{rs}$\;
  \FOR{$i \in range(n_{rs})$}
  	\STATE initialize $\theta$ randomly\;
  	\FOR{$e \in range(epochs)$}
   		\STATE compute reward function based on $\theta$\;
   		\STATE run approximate value iteration with Algorithm \ref{alg:value}\;
   		\STATE run Bellman Gradient Iteration with Algorithm \ref{alg:gradient}\;
   		\STATE compute gradient  $\nabla L(\theta)$ with Equation \eqref{equation:loglikelihoodgradient}\;
   		\STATE gradient ascent: $\theta=\theta+learning\_rate*\nabla L(\theta)$;
  	\ENDFOR
  	\STATE compute reward function based on $\theta$\;
  	\STATE compute the log-likelihood based on the reward function\;
  \ENDFOR
  \STATE identify the reward function with the highest log-likelihood\;
  \STATE return the reward function.
\end{algorithmic}
\end{algorithm}

\begin{algorithm}[tb]
  \caption{Approximate Value Iteration}
  \label{alg:value}
\begin{algorithmic}[1]

  \STATE Data: {$S,A,P,R,\gamma$,k}
  \STATE Result: {optimal value $V[S]$, optimal action value $Q[S,A]$}
  \STATE assign $V[S]$ arbitrarily
  \WHILE{$diff>threshold$}
	\STATE initialize $V'[S]=\{0\}$
    \FOR{$s\in S$}
      \STATE initialize $T[A]=\{0\}$
     \FOR{$a\in A$}
      \STATE $T[a]=\sum_{s'\in S}P_{ss'}^a(R[s']+\gamma*V[s'])$
      \ENDFOR
     \STATE $V'[s]=apprxMax(T[A],k)$
    \ENDFOR
    \STATE $diff=abs(V[S]-V'[S])$
    \STATE $V[S]=V'[S]$
  \ENDWHILE
  \STATE initialize $Q[S,A]=\{0\}$
  \FOR{$s\in S$}
    \FOR{$a\in A$}
      \STATE $Q[s,a]=Q[s,a]+\sum_{s'\in S}P_{ss'}^a(R[s']+\gamma*V[s'])$
      \ENDFOR
     \ENDFOR
\end{algorithmic}
\end{algorithm}

\begin{algorithm}[tb]
  \caption{Bellman Gradient Iteration}
  \label{alg:gradient}
\begin{algorithmic}[1]

  \STATE Data: {$S,A,P,R,V,Q,\gamma$,k}
  \STATE Result: {value gradient $V_G[S]$, Q-value gradient $Q_G[S,A]$}
  \STATE assign $V_G[S]$ arbitrarily\;
  \WHILE{$diff>threshold$}
    \STATE initialize $V'_G[S]=\{0\}$\;
    \FOR{$s\in S$}
      \STATE initialize $T_G[A]=\{0\}$\;
     \FOR{$a\in A$}
       \STATE $T_G[a]=\frac{\partial apprxMax}{\partial Q[s,a]}\sum_{s'\in S}P_{ss'}^a(\frac{\partial
       R[s']}{\partial \theta}+\gamma*V_G[s'])$\;
     \ENDFOR
     \STATE $V_G'[s]=\sum T_G[A]$\;
   \ENDFOR
    \STATE $diff=abs(V_G[S]-V'_G[S])$\;
    \STATE $V_G[S]=V'_G[S]$\;
  \ENDWHILE
  \STATE initialize $Q_G[S,A]=\{0\}$\;
  \FOR{$s\in S$}
    \FOR{$a\in A$}
      \STATE $Q_G[s,a]=Q_G[s,a]+\sum_{s'\in S}P_{ss'}^a(\frac{\partial R[s']}{\partial \theta}+\gamma*V_G[s'])$
    \ENDFOR
  \ENDFOR
\end{algorithmic}
\end{algorithm}

\section{Experiments}
\label{irl::experiments}
We evaluate the proposed method in two simulated environments. 

The first example environment is a parking space behind a store, as shown in Figure
\ref{fig:environment}. A mobile robot tries to figure out the location of the exit by observing the
motions of multiple agents, like cars. Assuming that the true exit is in one corner of the space, we
can describe it with the gridworld mdp \cite{irl::irl1}.  In this $N\times N$ grid, the rewards for
all states equal to zero, except for the upper-right corner state, whose reward is one,
corresponding to the true exit, as shown in Figure \ref{fig:gridworld}.  Each agent starts from a
random state, and chooses in each step one of the following actions: up, down, left, and right.
Some trajectories are shown in Figure \ref{fig:environmenttrajectories}.  Each action has a 30\%
probability that a random action from the set of actions is actually taken. We use a linear
function to represent the reward, where the feature of a state is a length-$N^2$ vector indicating
the position of the grid represented by the state, e.g., the $i_{th}$ element of the feature vector
for the $i_{th}$ state equals to one and all other elements are zeros. 

\begin{figure}
  \centering
  \subfloat[A testing environment: in the encircled space, only one exit exists, but the mobile robot
  can only observe the space within the dashed lines, and it has to observe the motions of cars,
  shown as black dots in the figure, to estimate the location of the
  exit.]{\includegraphics[width=0.18\textwidth]{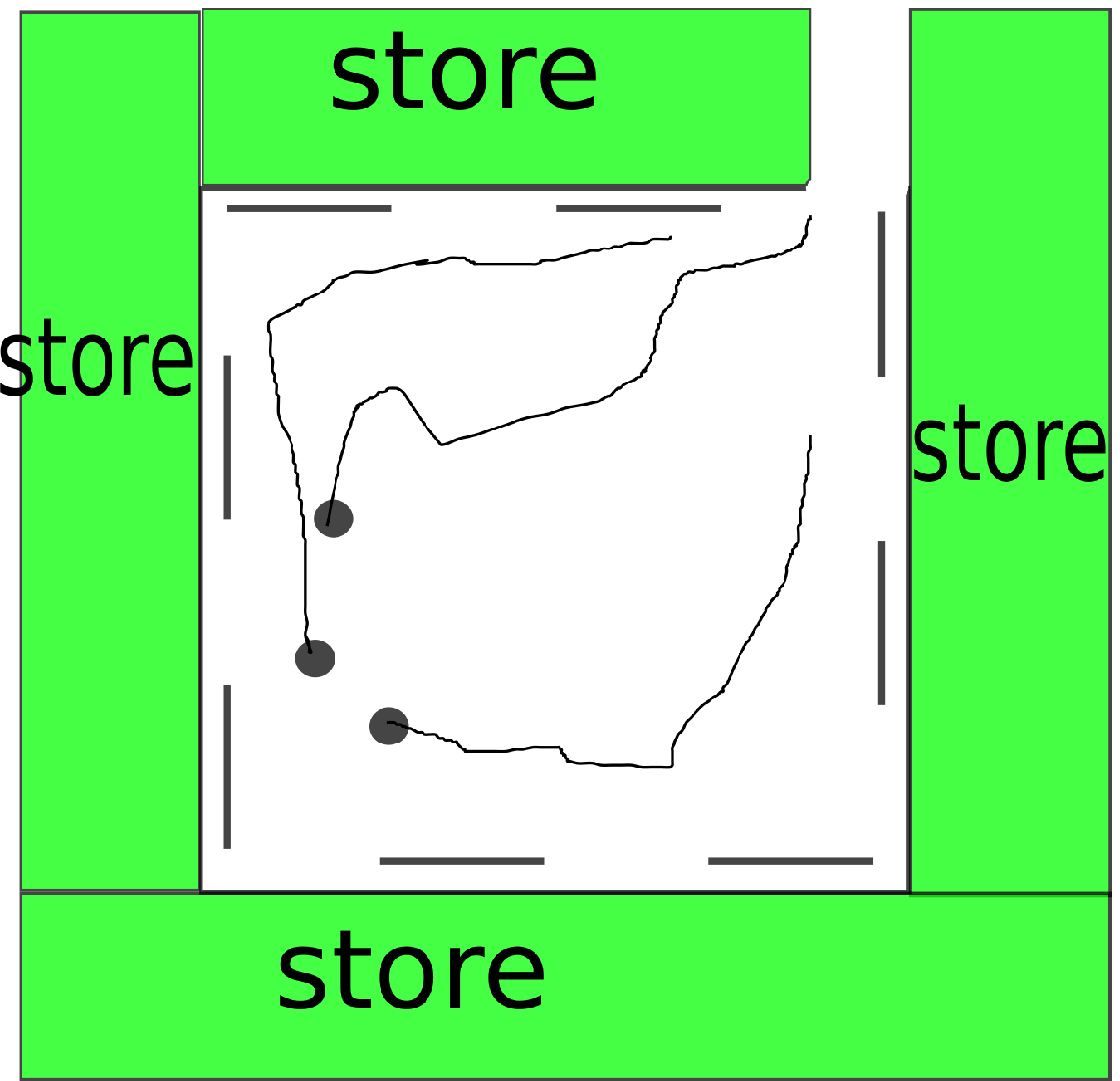}
  \label{fig:environment}}
  ~
  \subfloat[Example trajectories in Gridworld MDP: each agent starts from a random position, and
  follows an optimal policy to approach the exit. The black dots represent the initial positions of
  the agents. Each colored path denotes one trajectory with finite
  length.]{\includegraphics[width=0.22\textwidth]{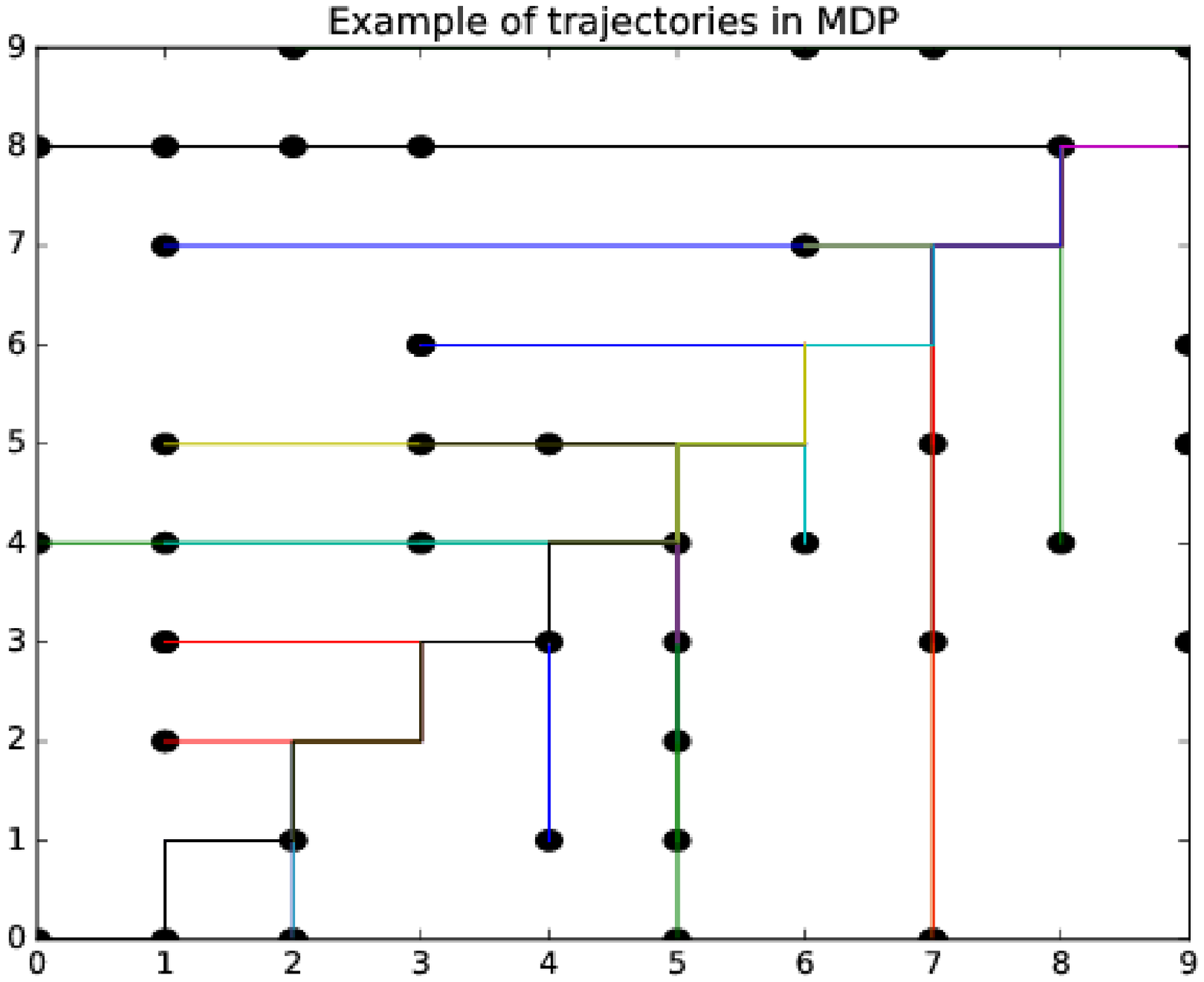}
  \label{fig:environmenttrajectories}}
\caption{A simulated environment}
\end{figure}

\begin{figure}
  \centering
  \subfloat[A reward table for the gridworld mdp on a $10\times 10$ grid.]{
    \includegraphics[width=0.2\textwidth]{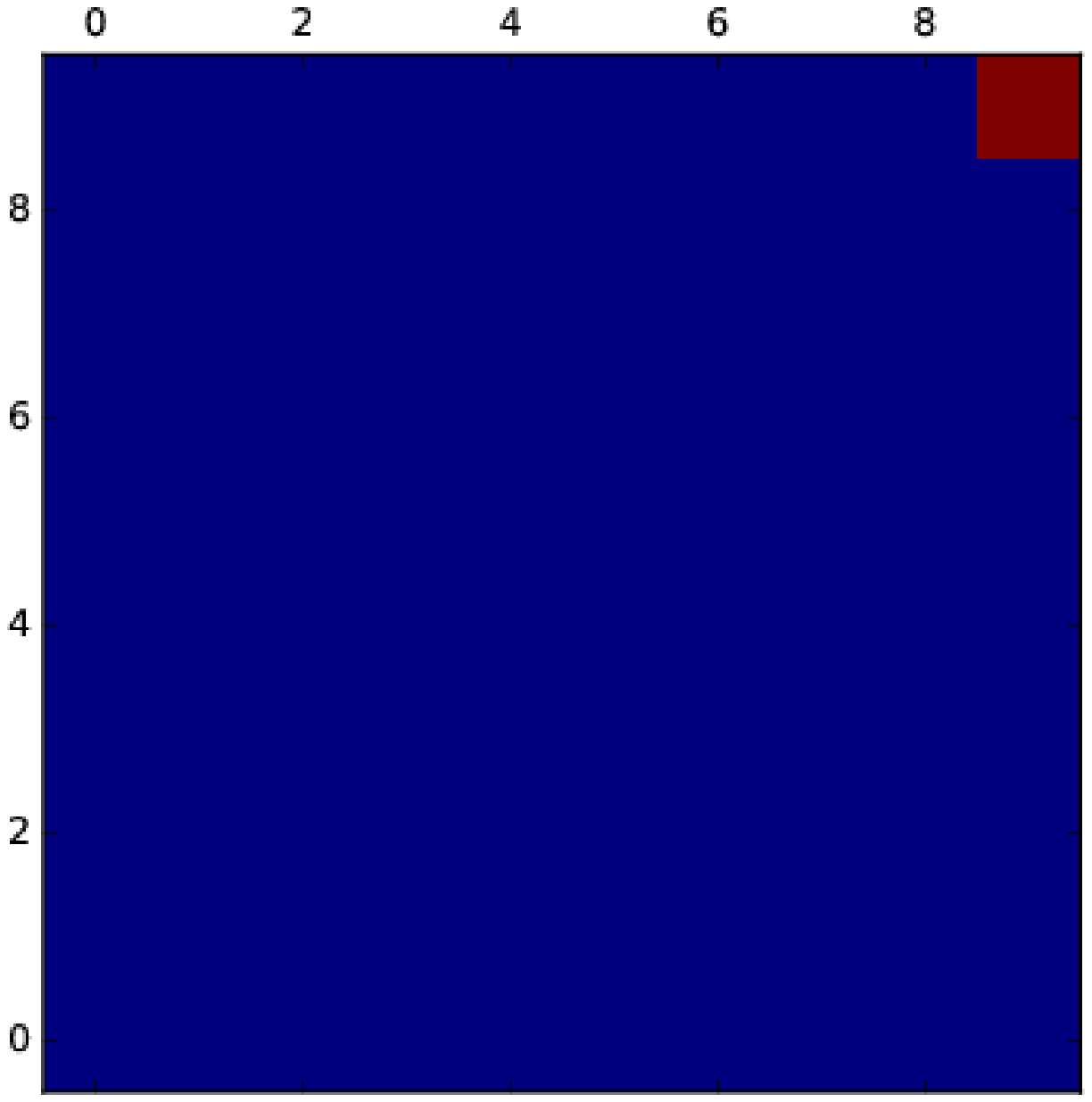}
    \label{fig:gridworld}
  }
  ~
  \subfloat[An example of a reward table for one objectworld mdp on a $10\times 10$ grid: it depends
  on randomly placed objects.]{
  \includegraphics[width=0.2\textwidth]{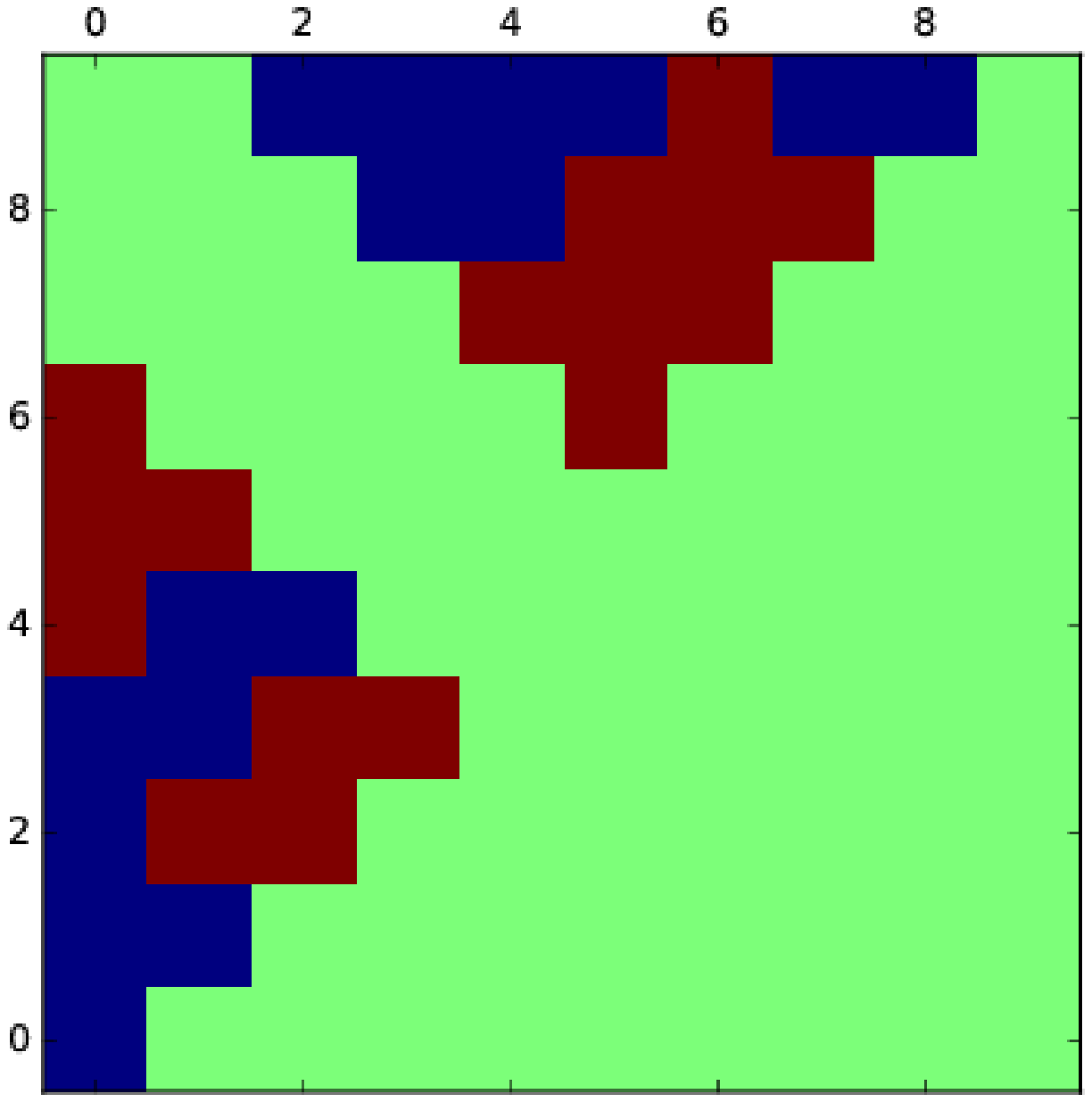}
  \label{fig:objectworld}
}
\caption{Examples of true reward tables}
\end{figure}

The second environment is an objectworld mdp \cite{irl::gaussianirl}. It is similar to the gridworld
mdp, but with a set of objects randomly placed on the grid. Each object has an inner color and an
outer color, selected from a set of possible colors, $C$. The reward of a state is positive if it is
within 3 cells of outer color $C1$ and 2 cells of outer color $C2$, negative if it is within 3 cells
of outer color $C1$, and zero otherwise. Other colors are irrelevant to the ground truth reward. One
example is shown in Figure \ref{fig:objectworld}. In this work, we place two random objects on the
grid, and use a linear function to represent the reward, where the feature of a state indicates its
discrete distance to each inter color and outer color in $C$. The true reward is nonlinear.

In each environment, the robot's trajectories are generated based on the true reward function.

\subsection{Qualitative Results}
We show some qualitative results with the proposed methods on 50 randomly generated trajectories,
where each trajectory has a random start and 10 steps.

For the p-norm approximation, we manually choose several parameter settings. In each of the
parameter settings, we run the algorithm fifty times with random reward parameter initializations,
and compare their log-likelihood values. For the parameter leading to the highest log-likelihood, we
compute the reward table. Several comparisons of ground-truth rewards and learned rewards are shown
in Figures \ref{fig:dem10} and \ref{fig:dem1}. For the g-soft approximation, we follow the same
procedure, and the results are shown in Figures \ref{fig:dem20} and \ref{fig:dem11}.

\begin{figure}
  \centering
  \subfloat[Demonstration of the p-norm method on a gridworld mdp: $k=100,b=1,e=1000,\alpha=0.01$.
  Left: ground truth. Right: recovered reward functin.]{
    \includegraphics[width=0.2\textwidth]{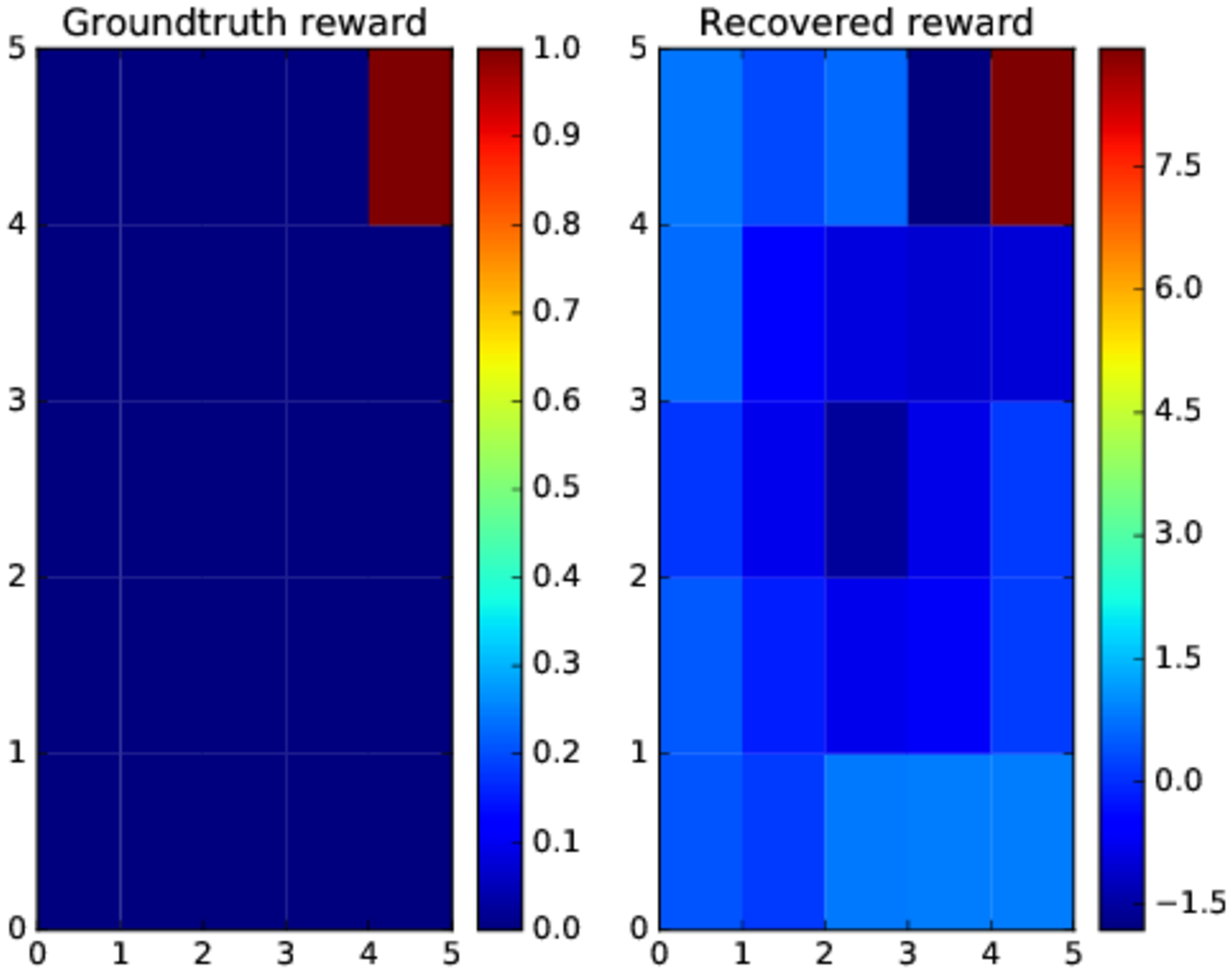}
    \label{fig:dem10}
  }
  ~
  \subfloat[Demonstration of the g-soft method on a gridworld mdp:
  $k=10,b=1,e=1000,\alpha=0.01$.Left: ground truth. Right: recovered rewrad function.]{
  \includegraphics[width=0.2\textwidth]{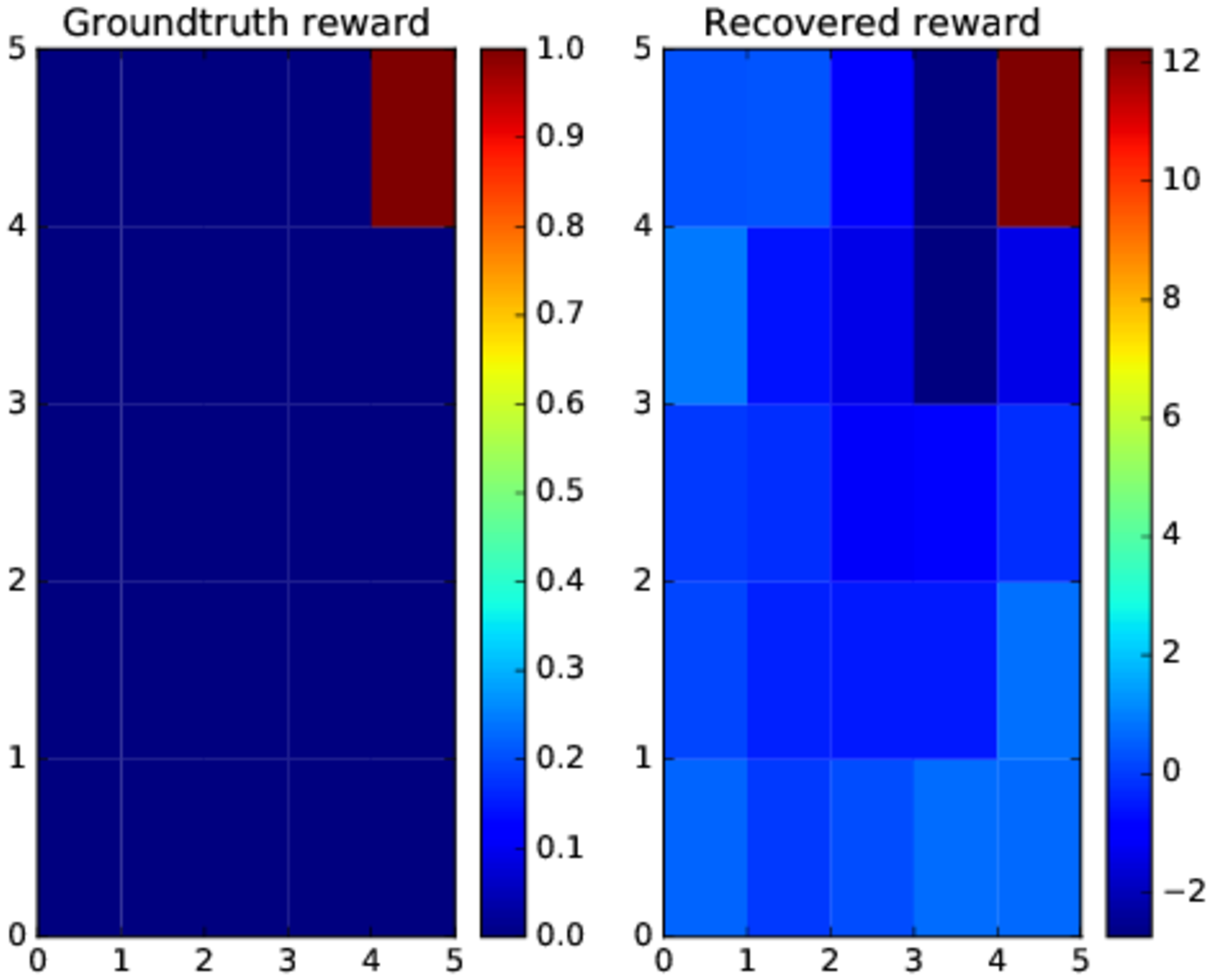}
  \label{fig:dem20}
}
\caption{Reward learning on gridworld mdp}
\end{figure}

\begin{figure}
  \centering
  \subfloat[Demonstration of the p-norm method on an objectworld mdp:
  $k=30,b=1,e=1000,\alpha=0.01$. Left: ground truth. Right: recovered reward function.]{
    \includegraphics[width=0.2\textwidth]{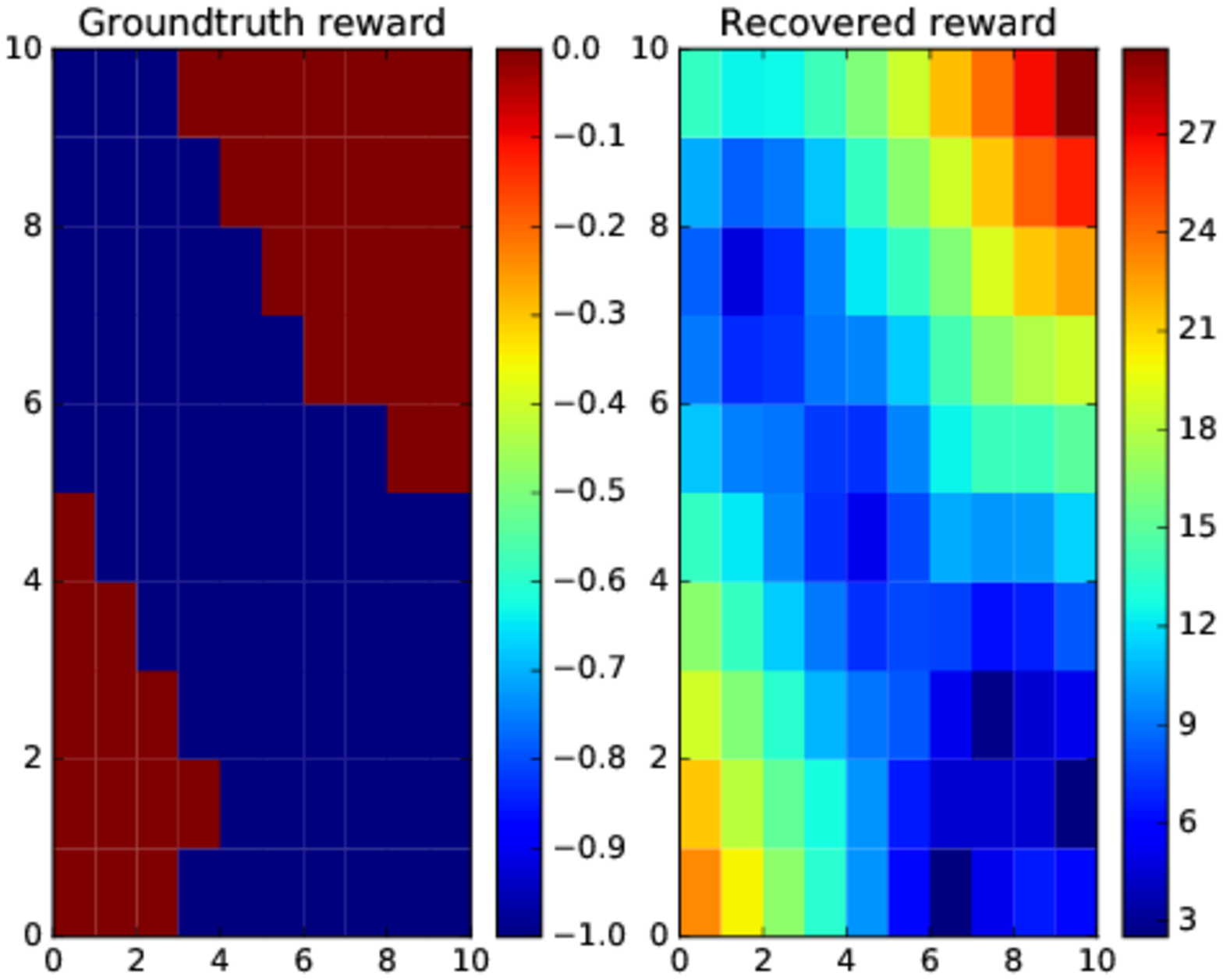}
    \label{fig:dem1}
  }
  ~
  \subfloat[Demonstration of the g-soft method on an objectworld mdp:
  $k=0.5,b=1,e=1000,\alpha=0.01$. Left: ground truth. Right: recovered reward function.]{
  \includegraphics[width=0.2\textwidth]{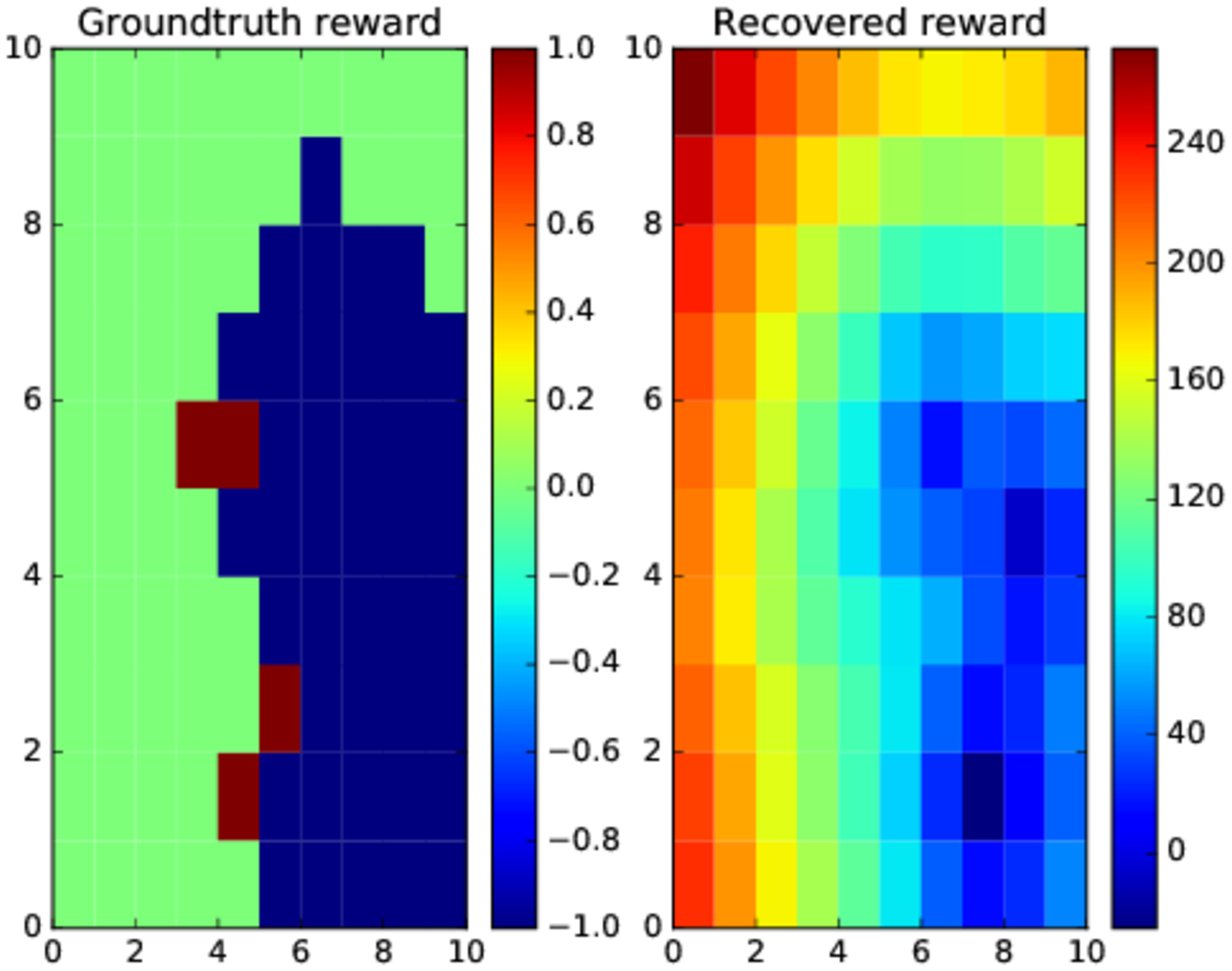}
  \label{fig:dem11}
}
\caption{Reward learning on objectworld mdp}
\end{figure}

\subsection{Quantitative Results}
\begin{figure}
  \centering
  \includegraphics[width=0.4\textwidth]{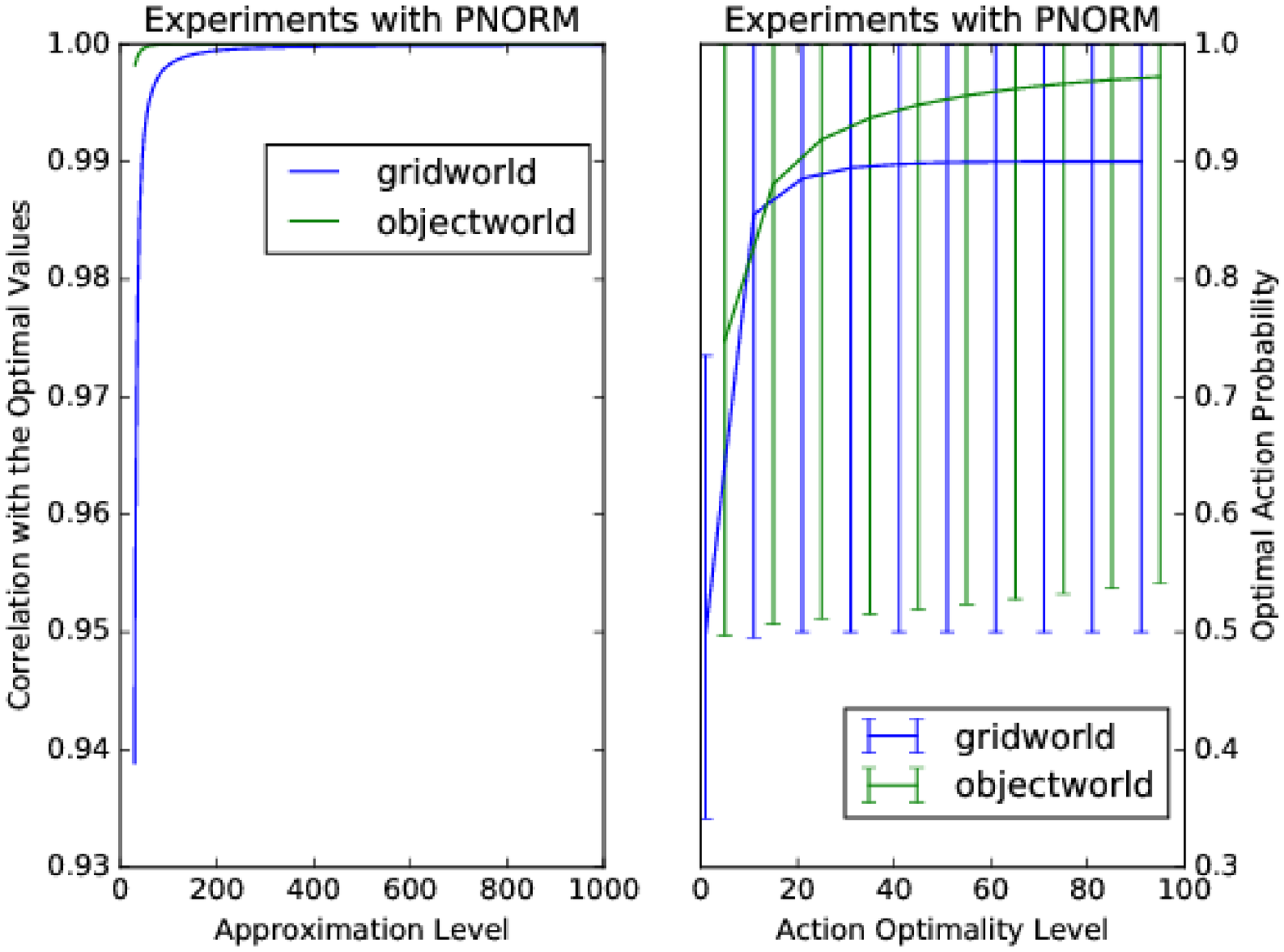}
  \caption{The effect of different approximation levels $k$ and confidence level $b$ on
  optimal value and optimal action selection with p-norm approximation: in both environments, when
  $k>500$, the approximated value function is nearly identical to the optimal value. For $b$, the
  probability to choose the optimal action keeps increasing in objectworld, but remains smaller than
  0.9 in grid world. }
  \label{fig:kbpnormresult}
\end{figure}

\begin{figure}
  \centering
  \includegraphics[width=0.4\textwidth]{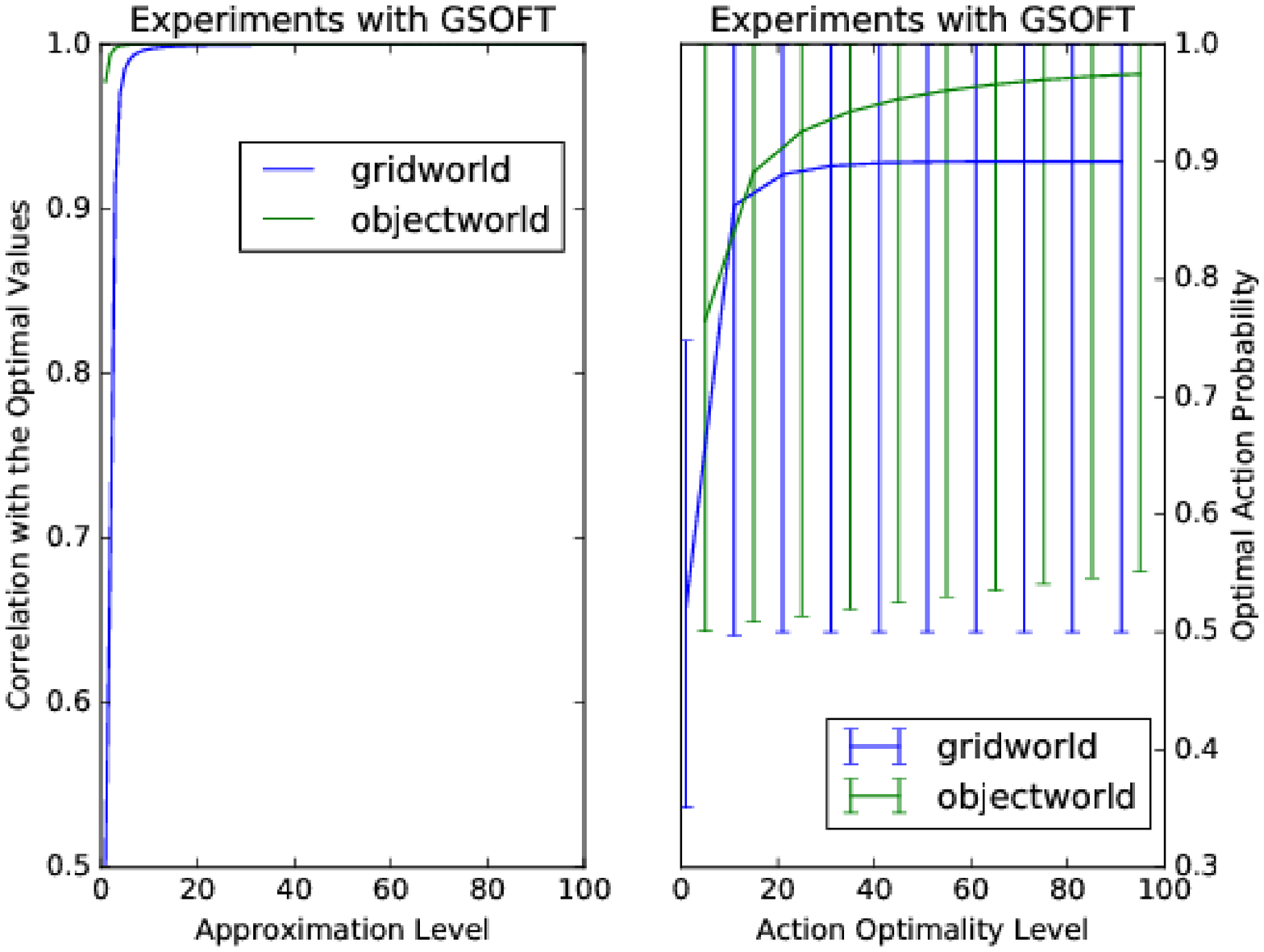}
  \caption{The effect of different approximation levels $k$ and confidence level $b$ on
  optimal value and optimal action selection with g-soft approximation: in both environments, when
  $k>20$, the approximated value function is nearly identical to the optimal value. For $b$, the
  probability to choose the optimal action keeps increasing in objectworld, but remains smaller than
  0.9 in grid world. }
  \label{fig:kbgsoftresult}
\end{figure}

\begin{figure}
  \centering
  \includegraphics[width=0.4\textwidth]{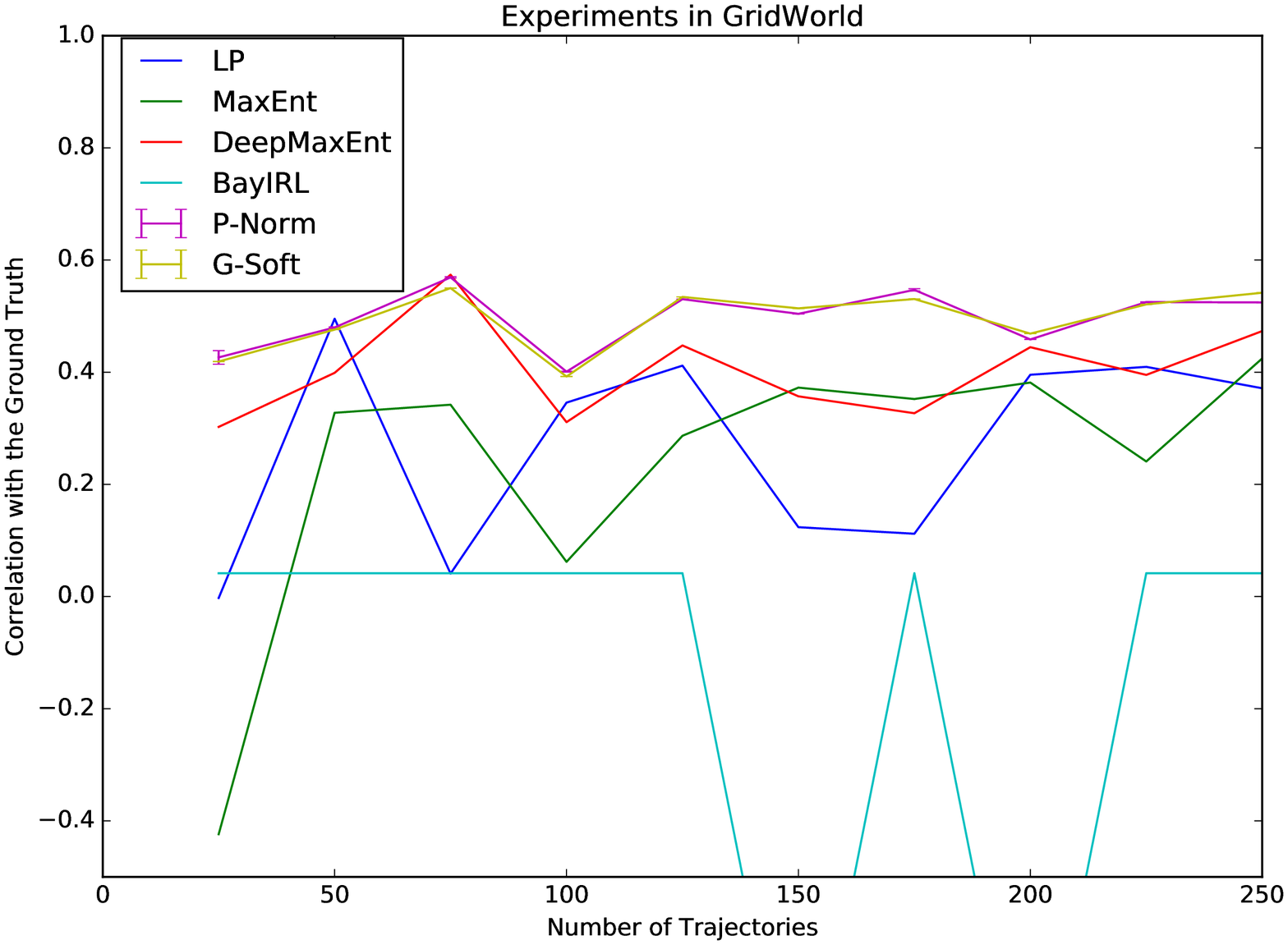}
  \caption{Comparison in gridworld: for each set of trajectories, we run the proposed method 100 
  times, each time with a random initial parameter, and compute the correlation coefficient between
  the learned reward and the true reward. The correlation coefficients for linear
  programming (LP), maximum entropy (MaxEnt), Bayesian IRL(BayIRL), and deep learning (DeepMaxEnt)
  are plotted. For the proposed method (PNORM approximation and GSOFT approximation), the
  correlation coefficient of the reward function with the highest log-likelihood and the standard
  deviation of the coefficients under random initial parameters are plotted.} 
  \label{fig:gridcompareresult}
\end{figure}

\begin{figure}
  \centering
  \includegraphics[width=0.4\textwidth]{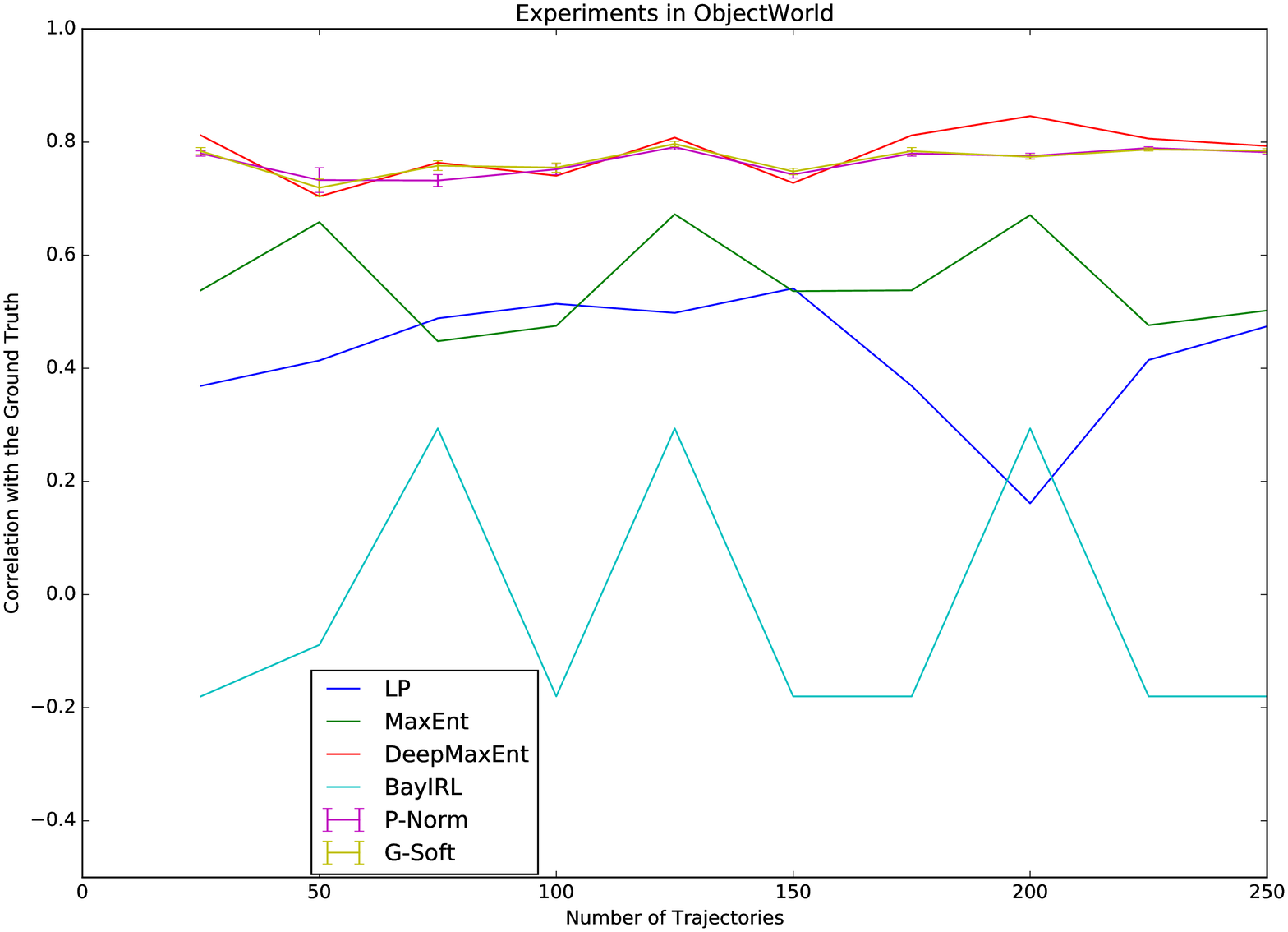}
  \caption{Comparison in objectworld: for each set of trajectories, we run the proposed method 100 
  times, each time with a random initial parameter, and compute the correlation coefficient between
  the learned reward and the true reward. The correlation coefficients for linear programming (LP),
  maximum entropy (MaxEnt), Bayesian IRL(BayIRL), and deep learning (DeepMaxEnt) are plotted. For
  the proposed method (PNORM approximation and GSOFT approximation), the correlation coefficient of
  the reward function with the highest log-likelihood and the standard deviation of the coefficients
  under random initial parameters are plotted.} 
 \label{fig:objectcompareresult}
\end{figure}

We evaluate the proposed method in three aspects: the accuracy of the value function approximation,
a comparison of the proposed method with existing methods, and the scalability of the proposed
method to large state space. We change the two environments to $5\times5$ grids to reduce the
computation time, but the dimension of the feature vector is still high enough to make the reward
function complex. The manually selected parameters for reward learning are the number of iterations
$e=1000$, the learning rate $\alpha=0.001$, and the discount factor $\gamma=0.9$. The parameters to
be evaluated include the approximation level $k$ and the confidence level $b$.

First, we run the approximate value iteration algorithm and the motion model in Equation
\eqref{equation:motionmodel} with different values of $k$ and $b$ in two environments. To evaluate
the approximation level $k$, we set the range of $k$ as 30 to 1000 for p-norm approximation, and 1
to 100 for g-soft approximation, and set $b=1$.  For each $k$, we compute the approximate value
function and evaluate it based on the correlation coefficient between the approximate value function
and the optimal value function. To evaluate the confidence level $b$, we choose the
range of $b$ as 1 to 100 for both approximations, and $k=1000$ for p-norm approximation, $k=100$ for
g-soft approximation. For each $b$, we compute the Q-function and the probability to take the
optimal action in each state. With only $5\times 5$ states, we compute the exact minimum, maximum,
and mean value of the probabilities in all states to take the optimal actions. The results are shown
in Figures \ref{fig:kbpnormresult} and \ref{fig:kbgsoftresult}.

The figures show that with sufficiently large $k$, the approximate value iteration generates almost
the identical result as the exact calculation. Therefore, to compute the gradient of the optimal
value with respect to a reward parameter, we can choose the largest $k$ that does not lead to data
overflow. However, the situation is different for $b$. Although the mean probability of optimal
actions increases with larger $b$, the mean value of the probabilities in all states to take the
optimal actions is always smaller than 0.9 in gridworld, because many state-action pairs have the
same Q-values, leading to multiple optimal policies, and the probability for each policy is always
smaller than 1.

Second, we compare the proposed method with existing methods, including the linear programming (LP)
approach in \cite{irl::irl1}, Bayesian method (BayIRL) in \cite{irl::bayirl}, the maximum entropy
(MaxEnt) approach in \cite{irl::maxentropy}, and a latest method based on deeep learning
(DeepMaxEnt) in \cite{irl::deepirl}. We randomly generate different numbers of trajectories, ranging
from 25 trajectories to 250 trajectories, and run the proposed method 100 times on the data, each
with a random initial parameter. The learned rewards are evaluated based on the correlation
coefficient with the true reward function. For existing methods, we compute the correlation
coefficient, and for the proposed method, we compute the correlation coefficient of the reward
function associated with the highest log-likelihood. To evaluate the multi-start strategy, we also
plot the standard deviation of the correlation coefficients. The parameter for p-norm approximation
is $k=100,b=1$ in both environments, and the parameter for g-soft approximation is $k=10,b=1$ in
both environments. Other parameters are shared among all methods. The comparison results are plotted
in Figures \ref{fig:gridcompareresult} and \ref{fig:objectcompareresult}.

The results show that in gridworld, where the ground truth is a linear reward function, the proposed
method performs better than existing methods, and only DeepMaxent outperforms the proposed method
occasionally. In objectworld, where the ground truth is a non-linear reward function, the proposed
method is second to DeepMaxent, because it adopts a non-linear neural network to model the reward
function while the proposed method uses a linear function. Besides, the theoretically
locally optimal results are quite similar to each other, because under a linear reward function and
a large approximate level $k$, the approximated Q values are approximately linear and the objective
function is close to a convex function.

Third, we test the scalability of the proposed method. We change the number of states in the
objectworld environment, and test the amount of time needed for one iteration of gradient ascent. To
record the accurate time, we do not adopt any parallel computing, and run the proposed method on a
single core of Intel CPU i7-6700. The implementation is a mix of C and python. The result is given
in Table \ref{table:scale}.

\begin{table}
  \caption{Computation time (second) for one iteration of gradient ascent.}
  \begin{tabular}{|c|c|c|c|c|c|c|}
    \hline
    state size&25&100&400&1600&6400&14400\\\hline
    pnorm&0.007&0.112&2.570&58.588&1560.014&8689.786\\\hline
    gsoft&0.004&0.093&2.088&54.136&1398.481&7035.641\\\hline
  \end{tabular}
  \label{table:scale}
\end{table}

The result shows that the algorithm can run a fair number of states, and in practice, the method can
be easily implemented as an efficient parallel algorithm by converting the Bellman Gradient
Iteration into matrix operations. Another bottleneck of the method is fitting the transition
model into the memory, whose size is $O(N_S^2*N_A)$, but in practice, we may divide it into
sub-matrices for efficiency

In summary, with a proper motion model to describe the actions, the proposed method performs better
than existing methods under linear reward functions while comparable to the state-of-the-art method
based on deep neural network. Besides, the proposed method is defined on state-action pairs, instead
of trajectories of fixed length, and this provides great flexibility in modeling practical actions.

Two minor drawbacks of the proposed methods are the locally optimal results and the
resource-intensive computation in large state space. But in practice, an approximately global
optimum can be achieved with a sufficiently high approximation level. The computation problem can be
solved with parallel computing on multi-core CPU or GPU. The major drawback of the proposed method
is the assumption of a known environment dynamics. The problem may be solved by sampling the motion
trajectories and estimating the dynamics.

\addtolength{\textheight}{-7.5cm}
\section{Conclusions}
\label{irl::conclusions}
This work introduced two approximations of the Bellman Optimality Equation to model the relation
between action selection and reward function in a differentiable way, and proposed a Bellman
Gradient Iteration method to efficiently compute the gradient of Q-value with respect to reward
functions.  This method allows us to learn the reward with gradient methods and model different
behaviors by varying the approximation level. We test the proposed method in two simulated
environments, and reveal how different parameter settings affect the accuracy of reward learning. We
compare the proposed method with existing approaches, and show that the proposed method is more
accurate and flexible in learning reward functions from the observed actions 

In future work, we will extend the proposed framework in multiple directions. First, we will search
for other approximation methods that lead to a concave Q-function, thus a global optimum can be
found. Second, we will apply the proposed method to other scenarios with different motion models,
like online learning for human motion analysis and deep learning for nonlinear reward functions.


\end{document}